\documentclass{article}

\PassOptionsToPackage{numbers, compress}{natbib}
\usepackage[preprint]{nips_2018}

\usepackage[utf8]{inputenc} 
\usepackage[T1]{fontenc}    
\usepackage{hyperref}       
\usepackage{url}            
\usepackage{booktabs}       
\usepackage{amsfonts}       
\usepackage{nicefrac}       
\usepackage{microtype}      
\usepackage{amsmath}
\usepackage{color}
\usepackage{csquotes}

\usepackage{graphicx}
\graphicspath{ {images/} }
\usepackage{scrextend}
\usepackage{booktabs}
\usepackage{float}
\usepackage{listings}

\usepackage{pgfplots}
\pgfmathsetseed{1}
\usepackage{tikz}
\usetikzlibrary{arrows.meta} 
\usetikzlibrary{calc} 
\usetikzlibrary{math} 
\usetikzlibrary{decorations}
\usetikzlibrary{decorations.pathmorphing}
\usepackage{tkz-euclide}
\usetkzobj{all}

\usepackage{amsthm}
\newtheorem{theorem}{Theorem}

\newtheorem{definition}{Definition}

\usepackage{algorithm}
\usepackage{algorithmicx}
\usepackage[noend]{algpseudocode}

\makeatletter
\def\BState{\State\hskip-\ALG@thistlm}
\makeatother

\tikzstyle{block} = [draw,fill=blue!20,minimum size=2em]
 
\tikzstyle{branch}=[fill,shape=circle,minimum size=3pt,inner sep=0pt]

\definecolor{xuxi}{rgb}{0.59, 0.0, 0.09}
\definecolor{josh}{rgb}{0.0, 0.42, 0.24}
\definecolor{peng}{rgb}{0.0, 0.0, 1.0} 
\definecolor{new}{rgb}{1.0, 0.49, 0.0}
\definecolor{old}{rgb}{0.25, 0.25, 0.25}
\definecolor{proposed}{rgb}{1.0, 0.49, 0.0}
\definecolor{accepted}{rgb}{0.0, 0.0, 0.0}

\newcommand{\exact}{\texttt{Exact} }
\newcommand{\memless}{\texttt{Memoryless} }

\title{Explainable Deterministic MDPs}

\author{
  Joshua R. Bertram    \quad    Peng Wei \\
  Iowa State University\\
  Ames, IA 50011 \\
  \texttt{\{bertram1, pwei\}@iastate.edu} \\
}

\begin{document}

\maketitle

\begin{abstract}
    We present a method for a certain class of Markov Decision Processes (MDPs) that can relate the optimal policy back to one or more reward sources in the environment.   For a given initial state, without fully computing the value function, q-value function, or the optimal policy the algorithm can determine which rewards will and will not be collected, whether a given reward will be collected only once or continuously, and which local maximum within the value function the initial state will ultimately lead to.  We demonstrate that the method can be used to map the state space to identify regions that are dominated by one reward source and can fully analyze the state space to explain all actions.  We provide a mathematical framework to show how all of this is possible without first computing the optimal policy or value function.
\end{abstract}

\section{Introduction}
	Markov Decision Processes (MDPs) are a framework for making decisions with broad applications to financial, robotics, operations research and many other domains.  If a problem can be formulated as an MDP, various algorithms can be used to solve the MDP resulting in what is known as an optimal policy.  An MDP is explainable in the sense that the optimal policy selects an action from a given state that will lead to the highest future expected reward.  This understanding is based on Bellman's well known dynamic programming work from \cite{bellman1957dynamic}.  However, beyond that general explanation why the MDP solution chooses a specific action at a specific time is opaque.  
    
    MDPs are formulated as the tuple $S, A, R, T $ where $S$ is the state at a given time $t$, $A$ is the action taken by the agent at time $t$ as a result of the decision process, $R$ is the reward received by the agent as a result of taking the action, and $T(s, a, s')$ is a transition function that describes the dynamics of the environment and capture the probability $p( s' | s, a )$ of transitioning to a state $s'$ given the action $a$ taken from state $s$.  
    
    An MDP is said to be deterministic if there is no uncertainty or randomness on the transition between $s$ and $s'$ given $a$; that is, $p( s' | s, a ) = 1.0$.  The output of an MDP is termed a policy, $\pi$, which describes an action $a$ that should be taken at every state $s \in S$.  When an MDP is solved completely such that the policy is optimal, it is typically denoted as $\pi^*$.  The optimal policy has the property that it maximizes the expected cumulative reward from any initial starting state.  Alternatively, the MDP solution can also be viewed as a value function that describes the value of being at each state, or also as a $Q$-value function that describes the value of taking a specific action from a given state.  Given one representation, it is possible to recover the other representations.  We use the notation of $V$ for the value function and $V^*$ for the optimal value function.  MDPs which contain states which "terminate", meaning that once the state is reached, no further actions are taken.  In chess, for example, a terminating state would be checkmate.  In other problems there may be no natural terminating state.  Such problems are said to be continuous.  The reward function $R$ defines the reward that the agent receives for taking action $a$ from state $s$.  Reward functions can be based off only the state, $R(s)$, off the state and action, $R(s, a)$, and occasionally off the resulting next state, $R(s, a, s')$.  
    
    There are many well known methods for solving MDPs exactly including value iteration and policy iteration, which are iterative methods based on the dynamic programming approach proposed by Bellman \cite{bellman1957dynamic}.  These algorithms use a table-based approach to represent the state-action space exactly and iteratively converge to the optimal policy $\pi^*$ and corresponding value function $V^*$.  These table-based methods have a well known disadvantage that they quickly become intractable.  As the number of states and actions increases in number or dimension, the number of entries in the (multi-dimensional) table increases exponentially.  Many real-world problems quickly exhaust the resources of even high performing computers.
    
    This curse of dimensionality is typically overcome by resorting to various forms of approximation of the optimal value function or optimal policy, some of which also have convergence guarantees or bounds on the error.  Other techniques have focused on managing the size of the state space explosion through factorization or through aggregation and tiling.
    
    In \cite{bertramExact} and \cite{bertramMemoryless} Bertram proposes two algorithms to solve MDPs named \exact and \memless that treat an MDP as a graph and use the connectivity of the graph and the distance between nodes in the graph to solve an MDP very efficiently.  Bertram \exact has $O( |R|^2 \times |A|^2 \times |S|)$ time complexity and $O( |S| + |R| \times |A| )$ memory complexity.  Bertram \memless has $O( |R|^3 \times |A|^2)$ time complexity and $O( |R| \times |A| )$ memory complexity and has no dependency on the state space.  Where \exact generates a full table-based value function equivalent to value iteration, \memless uses a novel way to represent the value function as a list of the rewards in the order that they are processed, and can construct any part of the value function on-demand from this list.  Both \exact \cite{bertramExact} and \memless \cite{bertramMemoryless}  operate on a restricted class of MDPs:  deterministic, continuous MDPs with positive real rewards (based only on state and not on action) and require a fully-connected environment where it is possible to transition from any state to any other state in the space.  Both papers use a grid world as an example environment.  
    
    In this paper, we examine how the method can be extended to:
\begin{enumerate}
\item determine which rewards will and will not be collected
\item whether a given reward will be collected only once or continuously
\item which local maximum within the value function the initial state will ultimately lead to
\end{enumerate}

We also show how to create a map of the state space to identify regions that are dominated by one reward source and can fully analyze the state space to explain all actions.    We provide a mathematical framework to underpin the claims in this paper.

    Despite the limited class of MDPs which can be solved with this method, the method leads to interesting results.  If the method can be generalized to a broader class of MDPs, it can perhaps be more broadly applied.  
    
\section{Related Work}

Researchers in many fields have long sought interpretable models that humans can understand.  For example, in 1976 \cite{shortliffe2012computer} describes expert systems that provide explanations on medical diagnoses.  Examples of the more recent use of the term explainability are \cite{van2004explainable} and \cite{gunning2017explainable}.

The solutions to Markov Decision Processes are commonly understood by many sources \cite{yellowbook,suttonbarto,bertsekas1995dynamic,powell2007approximate,bellman1957dynamic}, as maximizing the expected (or future) reward.  Factored MDPs \cite{schuurmans2002direct,guestrin2003efficient} can be viewed as an attempt to explain a large MDP by dividing it into multiple smaller MDPs, which on balance yields a potentially more understandable MDP.  In \cite{hauskrecht1998hierarchical} a similar result is achieved with hierarchical MDPs which operate over subsets of the state space.  In contrast,  \cite{lakkaraju2017learning} uses an MDP to learn an interpretable set of decisions, which makes the output of the MDP much more understandable without clarifying the internals of the MDP itself.

To the authors knowledge this is the first work that is able to trace the policy directly to the rewards in this fashion.

\section{Review of Key Concepts}

See \cite{bertramExact} and \cite{bertramMemoryless} for more detail, but in summary the key insights were to describe an MDP in terms of a graph, to take advantage of known structure of the MDP, and to utilize discoveries on how the expected reward from multiple reward sources interplay and result in the value function.  \memless extends \exact by using a computational trick to eliminate the need to represent the value function explicitly and instead represents the value function as a list of peaks that are formed by the reward sources, and then demonstrates how to recover the value function from this list of peaks.

We repeat some key concepts from \cite{bertramExact} which will be extended; please refer to the paper for the formal mathematical definitions.   Rewards are collected either once or continuously (infinitely) under a given policy.  Peaks form in the value function where rewards occur in the state space.  Baseline peaks, denoted $\mathcal{B}$, are formed when a single reward source is collected continuously.  Combined peaks, denoted $\Gamma$, are formed when two or more reward sources are collected together continuously.  Delta peaks, denoted $\Delta$, are formed when a single reward source is collected only once along the way to either a baseline or combined peak.  Where baseline or combined peaks form, the repeating path that is followed after reaching the goal state for the first time is termed the minimum cycle and represents the path that is followed once the peak is reached.

\begin{figure}[H]
\centering
\begin{tikzpicture}[decoration={markings, mark=at position 1.0 with {\arrow{>}};}]
      \tkzDefPoint(2,.05){A}
      \tkzDefPoint(2,.3){B}
      \tkzDrawLines[postaction={decorate}](A,B)
      \tkzLabelPoint[above right](B){$r_d$}
      \tkzLabelPoint[below](A){$s_d$}
      
      \tkzDefPoint(4,.00){A}
      \tkzDefPoint(4,1.3){B}
      \tkzDrawSegments[postaction={decorate}](A,B)
      \tkzLabelPoint[above right](B){$r_i$}
      \tkzLabelPoint[below](A){$s_i$}
      
      \tkzDefPoint(7,.00){A}
      \tkzDefPoint(7,1.0){B}
      \tkzDrawSegments[postaction={decorate}](A,B)
      \tkzLabelPoint[above right](B){$r_p$}
      \tkzLabelPoint[below](A){$s_p$}
      
      \tkzDefPoint(8,.00){A}
      \tkzDefPoint(8,0.7){B}
      \tkzDrawSegments[postaction={decorate}](A,B)
      \tkzLabelPoint[above right](B){$r_s$}
      \tkzLabelPoint[below](A){$s_s$}

      \draw[->] (0 ,0) -- (10,0) node[right] {$state$};
      \draw[->] (0 ,0) -- (0,3) node[above] {$value$};
      \draw[thick,domain=0:10,samples=1000,variable=\x,blue] plot ({\x},{3*0.9^(abs(\x-4))});
      \draw[thick,domain=0:10,samples=1000,variable=\x,red] plot ({\x},{1.5*0.9^(abs(\x-7))+1*0.9^(abs(\x-8))});
      \draw[thick,domain=0:10,samples=1000,variable=\x,green] plot ({\x},{2.8*0.9^(abs(\x-2))});
      \draw[thick,densely dotted,domain=2.45:2.8,smooth,variable=\y,black]  plot (2,{\y});
      \tkzLabelPoint[above right](3.6,3  ){$\mathcal{B}^i$}
      \tkzLabelPoint[above right](6.6,2.5){$\Gamma^{p,s}$}
      \tkzLabelPoint[above right](1.6,2.8){$\Delta^d$}
      
\end{tikzpicture}
\caption{Illustration of baseline (blue), combined baseline (red), and delta baseline (green).  Courtesy of \cite{bertramExact}.}
\end{figure}
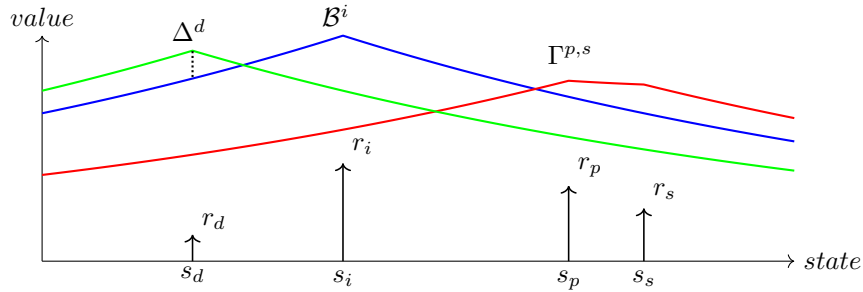

    In the proof of Bertram \exact value functions which are composed of one or more of these baseline functions are termed $V^{\mathcal{M}}$ and the set of all value functions is between the zero-function $V_{\emptyset}(s) = 0$ is termed $V^\alpha$.  It was proven that the optimal value function $V^*$ lies both within $V^\mathcal{M}$ and $V^\pi$ and is the maximum of $V^\mathcal{M}$, $V^\pi$, and $V^\alpha$.  The computational trick used in Bertram \memless is the realization that the value function $V^\mathcal{M}$ can be constructed from a given set of baseline functions $M \in \mathcal{M}$, which was used in \cite{bertramMemoryless} to eliminate the need for storing the value function as a table in memory.

\newcommand{\DrawBean}[6]{
 \tkzDefPoint(#1+#2*(0.4   ) ,#3*0.8   ){A}
 \tkzDefPoint(#1+#2*(1     ) ,#3*0.4   ){B}
 \tkzDefPoint(#1+#2*(1.8   ) ,#3*1.0   ){C}
 \tkzDefPoint(#1+#2*(2.3773) ,#3*1.4421){D}
 \tkzDefPoint(#1+#2*(2.6905) ,#3*2.1074){E}
 \tkzDefPoint(#1+#2*(2.3752) ,#3*2.8828){F}
 \tkzDefPoint(#1+#2*(1.4   ) ,#3*3.0   ){G} 
 \tkzDefPoint(#1+#2*(0.6   ) ,#3*2     ){H} 
 
 \tkzDefPoint(#1+#2*(1.0   ) ,#3*0.8   ){I} 
 \tkzDefPoint(#1+#2*(1.8   ) ,#3*1.6   ){J} 
 \tkzDefPoint(#1+#2*(1.2   ) ,#3*2.0   ){K} 
 \tkzDefPoint(#1+#2*(0.7   ) ,#3*1.3   ){L} 
 
 \tkzDefPoint(#1+#2*(1.4   ) ,#3*1.5   ){M} 
 \tkzDefPoint(#1+#2*(2.4   ) ,#3*2.1   ){N} 
 
 
 \tkzLabelPoint[above      ](I){#5}
 \tkzLabelPoint[above left](M){#6}
 \tkzLabelPoint[below left](N){#4}

 \node [color=red] at (M) {\textbullet};

 \draw plot [smooth cycle, tension=0.7] coordinates {
 (A) (B) (C) (D) (E) (F) (G) (H)
 };
 \draw plot [smooth cycle, tension=0.7] coordinates {
 (I) (J) (K) (L) 
 };
}

\newcommand{\DrawOther}[6]{
 \tkzDefPoint(#1+#2*(0.4   ) ,#3*0.8   ){A}
 \tkzDefPoint(#1+#2*(1     ) ,#3*0.4   ){B}
 \tkzDefPoint(#1+#2*(1.8   ) ,#3*1.0   ){C}
 \tkzDefPoint(#1+#2*(2.3773) ,#3*1.4421){D}
 \tkzDefPoint(#1+#2*(2.6905) ,#3*2.1074){E}
 \tkzDefPoint(#1+#2*(2.3752) ,#3*2.8828){F}
 \tkzDefPoint(#1+#2*(1.4   ) ,#3*3.0   ){G} 
 \tkzDefPoint(#1+#2*(0.6   ) ,#3*2     ){H} 
 
 \tkzDefPoint(#1+#2*(0.2   ) ,#3*0.8   ){I} 
 \tkzDefPoint(#1+#2*(1.0   ) ,#3*1.4   ){J} 
 \tkzDefPoint(#1+#2*(0.3   ) ,#3*2.5   ){K} 
 \tkzDefPoint(#1+#2*(-.4   ) ,#3*1.3   ){L} 
 
 \tkzDefPoint(#1+#2*(0.6   ) ,#3*1.4   ){M} 
 \tkzDefPoint(#1+#2*(2.4   ) ,#3*2.1   ){N} 
 
 \tkzDefPoint(#1+#2*(-.9   ) ,#3*0.8   ){O}
 \tkzDefPoint(#1+#2*(0.8   ) ,#3*0.0   ){P}
 \tkzDefPoint(#1+#2*(1.8   ) ,#3*0.2   ){Q}
 \tkzDefPoint(#1+#2*(2.4773) ,#3*1.0421){R}
 \tkzDefPoint(#1+#2*(2.9905) ,#3*2.1074){S}
 \tkzDefPoint(#1+#2*(2.8752) ,#3*3.0828){T}
 \tkzDefPoint(#1+#2*(1.4   ) ,#3*3.4   ){U} 
 \tkzDefPoint(#1+#2*(0.3   ) ,#3*3     ){V} 
 
 
 \tkzLabelPoint[below right](L){#5}
 \tkzLabelPoint[above right](M){#6}
 \tkzLabelPoint[below left](N){#4}
 \tkzLabelPoint[below ](U){$V^\alpha$}

 \node [color=red] at (M) {\textbullet};

 \draw plot [smooth cycle, tension=0.7] coordinates {
 (A) (B) (C) (D) (E) (F) (G) (H)
 };
 \draw plot [smooth cycle, tension=0.7] coordinates {
 (I) (J) (K) (L) 
 };
 \draw plot [smooth cycle, tension=0.7] coordinates {
 (O) (P) (Q) (R) (S) (T) (U) (V)
 };
}

\def\oa{1}
\def\ob{9}
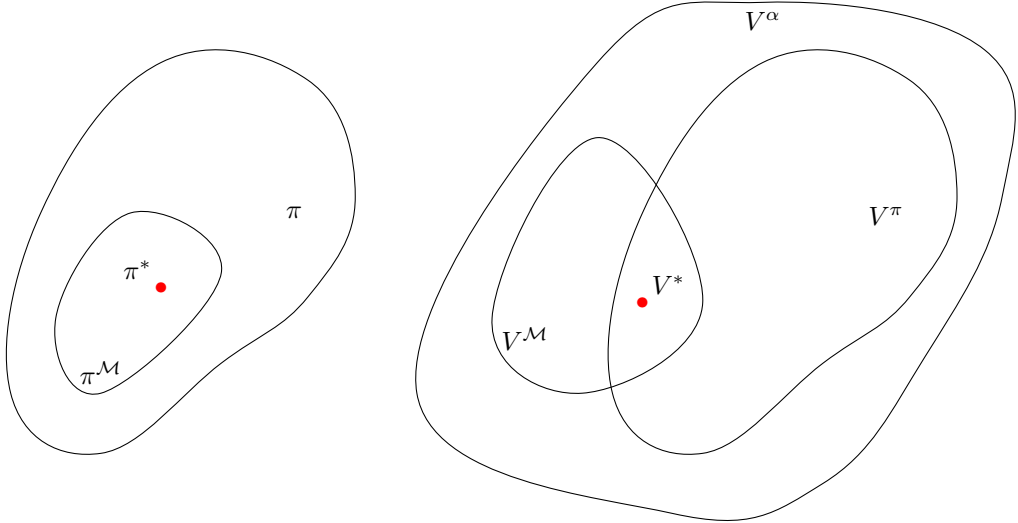
\begin{figure}[H]
\centering
\begin{tikzpicture}[label]
 \DrawBean{\oa}{2}{2}{$\pi$}{$\pi^{\mathcal{M}}$}{$\pi^*$}
 \DrawOther{\ob}{2}{2}{$V^\pi$}{$V^{\mathcal{M}}$}{$V^*$}
\end{tikzpicture}
\caption{Depiction of the relationship between policy, value function, and optimal solution for $V^{\mathcal{M}}$.  Courtesy of \cite{bertramExact}.}
\end{figure}

To construct the value function from the peaks, it was shown in \cite{bertramExact} that the value function $V^M$ is equal to the maximum value of each of the baselines at each state.

\begin{equation}
V^M(s) = \max_{M_i \in M} M_i(s), \forall s \in S
\end{equation}

We also bring the readers attention to the following theorem from \cite{bertramExact} (omitting the proof):

\begin{theorem}
All rewards $R = \{ r_1, r_2, \cdots, r_N \}$ are collected either once or infinitely under a given policy $\pi$.  That is, for a given reward $r_i \in R, N_k = \{ 1, \infty \}$, and only rewards falling within a minimum cycle of a local maximum in the value function are collected infinitely. 
\end{theorem}

Thus, wherever a minimum cycle occurs, rewards within that minimum cycle are collected infinitely, and the origin of the minimum cycle is a local maximum within the value function.

\section{Methodology}

In this paper, we use this list of peaks and extend the mathematical analysis in \cite{bertramExact} to show that baseline functions are sufficient to determine how the rewards will be collected.  As in \cite{bertramExact,bertramMemoryless}, we restrict our analysis to positive real rewards.

\subsection{Dominance}

First we show the following, which intuitively is the natural result of viewing the optimal policy as a "hill climb" through the value function: 

\begin{theorem}
For a fully connected MDP, the optimal policy always leads to a local maximum from every initial state.
\end{theorem}

\begin{proof}
From \cite{bertramExact}, it was shown that for a given policy all rewards are collected either once or infinitely, and that if a reward is collected infinitely, it is part of a minimum cycle that is a local maximum of the value function.

If we consider an initial state $s_i$, a set of rewards $R = \{ r_1, \cdots, r_N \}$, and the optimal policy $\pi^*$, we define the path taken through the state space by following the optimal policy as $\mathcal{K} = \{ s^{(1)}, s^{(2)}, \cdots,  s^{(k)} \}$, where $k$ represents the $k$-th step through the state space.  Note that for a continuous MDP, this path continues forever.  Let us denote the portion of this infinite path which leads to its maximum value as $\mathcal{K}^+ \subset \mathcal{K} | V( \mathcal{K}(i) ) < V( \mathcal{K}(j) ), \forall i \in \{ 1, \cdots, k \}, \forall j \in \{ i+1, \cdots, k \}$, where we label the maximum $i$ that satisfies the condition as $k_{max}$.  At each step $i \in \{ 1, \cdots, k_{max} \}$ of this path we may either collect a reward $r_n \in R$ or no reward.  If $i < k_{max}$, then we know by the definition of $k_{max}$ that $V( \mathcal{K}^+(i) ) < V( \mathcal{K}^+(i+1) )$ and that reward $r_n$ is collected only once (i.e., it is a delta reward).  If $i = k_{max}$, then we know that $V( \mathcal{K}(i+1) ) \leq V( \mathcal{K}(i) )$ and we have reached a local maximum in the value function where a minimum cycle must then form.  We denote the state at which the local maximum occurs as as $s_\mathcal{K}$.

Thus following the optimal policy must necessarily result in a path that leads ultimately to a local maximum in the value function.
\end{proof}

This proof also shows why delta peaks can never be local maximums, and that only baseline peaks and combined peaks can be local maximums.  Conversely, any baseline peak or combined peak that is selected by Bertram \exact or \memless is also a local maximum.  

We now define the concept of a \emph{dominant peak} which, informally, determines the local maximum that the optimal policy will guide an agent to from a given initial state.

\begin{definition}
From an initial state $s_i$, the \emph{dominant peak} is the peak located at $s_\mathcal{K}$ where the agent reaches the local maximum $s_\mathcal{K}$ along the optimal path $\mathcal{K}^+$ by following the optimal policy $\pi^*$.
\end{definition}

In the case of the two or more peaks that all have equal value at a state $s_i$, they are said to be \emph{co-dominant}.  The optimal policy at these points depends on how the policy extraction algorithm handles the case where multiple actions all lead to states with the same value.  Some implementations may deterministically choose, say, the lowest numbered action, while others may select an action randomly among multiple such actions.  Note that we only address a deterministic implementation in this paper.

By this definition, we know that we need not consider any delta peaks, as they are by definition collected once and cannot form a local maximum.  Given that we know the set of rewards $R$, and can determine the baseline peaks $\mathcal{B}$ and combined peaks $\Gamma$, how do we determine which of these candidates are the dominant peak at a given state $s_i$?

Recall the notation from \cite{bertramExact} of the propagation operator $\mathcal{P}$ which calculates the value function from a given peak.  The formal notation for a baseline peak's value function is
\begin{equation}
  \mathcal{P}_{\mathcal{B}^b}(s) = \gamma^{\delta(s,s_b)} \times \frac{r_b}{1 - \gamma^{\phi(s_b)}},
\label{eq:baseline}
\end{equation}
where $s_b$ is the state at which reward $r_b$ is collected, $\delta(s,s_b)$ is the distance from state $s$ to state $s_b$, and $\phi(s_b)$ is the minimum cycle distance for the MDP.

The formal notation for a combined peak's value function is defined as:
\begin{equation}
\mathcal{P}_{\Gamma^{p,s}}(s) = \mathcal{P}_{\mathcal{B}^p}(s) + \mathcal{P}_{\mathcal{B}^s}(s)
\label{eq:combined}
\end{equation}
where $s_p$ is the state at which primary reward $r_p$ is collected and $s_s$ is the state at which secondary reward $r_s$ is collected.

To evaluate the discounted future reward of a peak from a state $s_i$, we simply evaluate these value function definitions at $s_i$.

From \cite{bertramExact}, we know that the value function formed by any subset $M \in \mathcal{M}$ of peaks lies within $V^\mathcal{M}$ and that the value function $V^M \in V^\mathcal{M}$ formed by the peaks is determined by:
\begin{equation}
V^M(s) = \max_{M_i \in M} M_i(s), \forall s \in S
\end{equation}

Thus, at a given state $s_i \in S$, the value at the state is the maximum value of all the value functions within $M$ evaluated at state $s_i$.  Let us denote the set of peaks that form the value functions in $M$ as $P$.  Let us denote the peak with the maximum value at $s_i$ as $P_{max}$ and its value function as $M_{max}$, and let us define the subset of peaks which does not contain $P_{max}$ as $P_{sub} = P \setminus P_{max}$ and the corresponding subset of value functions as $M_{sub} = M \setminus M_{max}$.

Now let us consider any subset of the peaks $P$ that still contains the peak $P_{max}$, $P_{equiv} = \{ P_{max}, P_{sub}' \}$ where $P_{sub}' \subset P_{sub}$ and the corresponding value functions $M_{equiv} = \{ M_{max}, M_{sub}' \}$ where $M_{sub}' \subset M_{sub}$.  We note that the value of $M_{equiv}$ evaluated at $s$ remains the same as $M_{max}$ evaluated at $s$.  In fact, from the perspective of the agent at state $s$, the value function would remain the same even if the peaks with value functions in $P_{sub}$ were not present.  Thus, we say that $P_{max}$ \emph{dominates} the peaks in $P_{sub}$ at $s_i$ or that $P_{max}$ is the \emph{dominant peak} at $s_i$.

\begin{figure}[H]
\centering
\begin{tikzpicture}[decoration={markings, mark=at position 1.0 with {\arrow{>}};}]
      \tkzDefPoint(2,.05){A}
      \tkzDefPoint(2,.3){B}
      \tkzLabelPoint[below](A){$s_i$}
      
      \tkzDefPoint(5,.05){A}
      \tkzDefPoint(5,.3){B}
      \tkzLabelPoint[below](A){$s_j$}
      
      \tkzDefPoint(9,.05){A}
      \tkzDefPoint(9,.3){B}
      \tkzLabelPoint[below](A){$s_k$}
      
      \tkzDefPoint(4,.00){A}
      \tkzDefPoint(4,1.3){B}
      \tkzDrawSegments[postaction={decorate}](A,B)
      \tkzLabelPoint[above right](B){$r_b$}
      \tkzLabelPoint[below](A){$s_b$}
      
      \tkzDefPoint(7,.00){A}
      \tkzDefPoint(7,1.0){B}
      \tkzDrawSegments[postaction={decorate}](A,B)
      \tkzLabelPoint[above right](B){$r_p$}
      \tkzLabelPoint[below](A){$s_p$}
      
      \tkzDefPoint(8,.00){A}
      \tkzDefPoint(8,0.7){B}
      \tkzDrawSegments[postaction={decorate}](A,B)
      \tkzLabelPoint[above right](B){$r_s$}
      \tkzLabelPoint[below](A){$s_s$}

      \draw[->] (0 ,0) -- (10,0) node[right] {$state$};
      \draw[->] (0 ,0) -- (0,3) node[above] {$value$};
      \draw[thick,domain=0:10,samples=1000,variable=\x,blue] plot ({\x},{3*0.9^(abs(\x-4))});
      \draw[thick,domain=0:10,samples=1000,variable=\x,red] plot ({\x},{1.5*0.9^(abs(\x-7))+1*0.9^(abs(\x-8))});
      
      \draw[thick,densely dotted,domain=0:3,smooth,variable=\y,black]  plot (2,{\y});
      \draw[thick,densely dotted,domain=0:3,smooth,variable=\y,black]  plot (5,{\y});
      \draw[thick,densely dotted,domain=0:3,smooth,variable=\y,black]  plot (9,{\y});
      
      \tkzLabelPoint[above right](3.6,3  ){$\mathcal{B}^b$}
      \tkzLabelPoint[above right](6.6,2.5){$\Gamma^{p,s}$}
      
\end{tikzpicture}
\caption{Illustration of dominant peak.  At state $s_i$ and $s_j$, the peak $\mathcal{B}^b$ dominates $\Gamma^{p,s}$.  At state $s_k$, the peak  $\Gamma^{p,s}$ dominates $\mathcal{B}^b$.}
\end{figure}
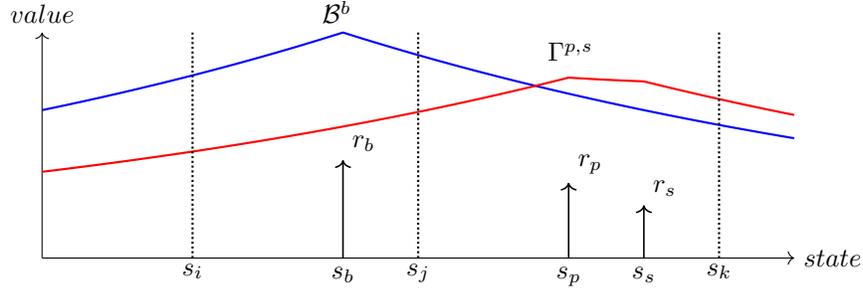

\begin{theorem}
If a peak $P_{max} \in P$ located at state $s_p$ is dominant at $s_i$, then the agent will reach state $s_p$ where the local maximum formed by the dominant peak.
\end{theorem}

\begin{proof}
Recall our definition of the optimal path $\mathcal{K}^+$ which describes the "hill climb" that is performed by following the optimal policy from state $s_i$.  

Let us denote the dominant peak at $s_i$ as $P_{max}^i$ and the corresponding value at $s_i$ as $V_{max}^i$.  Let us denote the value at $s_i$ of any other peak $p_{sub} \in P_{sub}$ as $V_{sub}^i$.

Let us consider what happens as we follow $\mathcal{K}^+$ when we take one step from $s_i$ to a next state $s_j$ and decrease the distance to $P_{max}$ by 1, we can say for certain that that the value of our peak $V_{max}$ at $s_j$ will increase compared to the value at $s_i$ due to the geometric progression of the discount factor:

\begin{equation}
\begin{split}
V_{max}( s_i ) &= \gamma \times V_{max}( s_j ) \\
V_{max}( s_j ) &= \frac{1}{\gamma} \times V_{max}( s_i ) 
\end{split}
\end{equation}

When we consider the change in the value of the other peak $p$ at $s_j$, we have three cases to consider.  The step from $s_i$ to $s_j$ may:

\begin{enumerate}
\item cause the distance to $p$ to increase by 1.  
In this case, $V_{sub}^i > V_{sub}^j$, and since $V_{max}^j > V_{max}^i$ and $V_{max}^i > V_{sub}^i$, then $V_{max}^j > V_{sub}^j$.  Therefore our dominant peak remains dominant at $s_j$.
\item cause the distance to $p$ to stay the same.  
In this case, $V_{sub}^i = V_{sub}^j$, and since $V_{max}^j > V_{max}^i$ and $V_{max}^i > V_{sub}^i$, then $V_{max}^j > V_{sub}^j$.  Therefore our dominant peak remains dominant at $s_j$.
\item cause the distance to $p$ to decrease by 1.
In this case, $V_{sub}^j > V_{sub}^i$, and in fact $V_{sub}^j = \frac{1}{\gamma} \times V_{sub}^i$.  We know that $V_{max}^i > V_{sub}^i$, so therefore: 
\begin{equation}
\begin{split}
\frac{1}{\gamma} \times V_{max}^i &> \frac{1}{\gamma} \times V_{sub}^i \\
 V_{max}^j  &> V_{sub}^j
\end{split}
\end{equation}
Therefore, our dominant peak remains dominant at $s_j$.
\end{enumerate}

By induction, this continues until we reach the end of $\mathcal{K}^+$, which we defined as the maximum of $\mathcal{K}$, where the local maximum lies and the minimum cycle occurs.

Therefore, we have proven that if a peak $P_{max}$ is dominant at initial state $s_i$, $\mathcal{K}^+$ will terminate at the local maximum formed by peak $P_{max}$.
\end{proof}

From this result, we have shown that from a given state $s_i$, we can determine the resulting local maximum we will be drawn to during the "hill climb" when following the optimal policy.  

If desired, we can therefore iterate over every state in the state space and determine the dominant peak, and from this information construct a map of the state space that shows the regions of the state space that are attracted to each peak.  We will describe this region as the \emph{region of dominance} for the corresponding dominant peak, and a state is said to lie within a \emph{dominated region} of a peak.

\newcommand\Rect[3]{
\filldraw[fill=#3!20!white, draw=black] (#1,#2) rectangle (#1+1,#2+1);
}

\newcommand\RArrow[2]{
  \draw[-{Triangle[length=4pt]}, line width=0.5mm]  (#1+.8,#2+.5) -- (#1+1+.2,#2+.5);
}

\newcommand\UArrow[2]{
  \draw[-{Triangle[length=4pt]}, line width=0.5mm]  (#1+.5,#2+.8) -- (#1+.5,#2+1+.2);
}

\begin{figure}[H]
\centering
\begin{tikzpicture}[node distance = 4cm, auto]
 \foreach \x in {0,...,5}
    \foreach \y in {0,...,5} 
        \filldraw[fill=white, draw=black] (\x,\y) rectangle (\x+1,\y+1);

\foreach \x in {0,...,5}
   \foreach \y in {0,...,5} 
      \Rect{\x}{\y}{green};
      
\foreach \x in {0,...,3}
   \foreach \y in {3,...,5} 
      \Rect{\x}{\y}{red};

\Rect{0}{2}{blue}
\Rect{0}{1}{blue}
\Rect{0}{0}{blue}
\Rect{1}{2}{blue}
\Rect{1}{1}{blue}

\Rect{3}{3}{green}

\node (r) at (4.5,1.5) {\textbf{$r_a$}};
\node (r) at (1.5,4.5) {\textbf{$r_b$}};
\node (r) at (0.5,2.5) {\textbf{$r_c$}};

\end{tikzpicture}
  \caption{Illustration of a map showing the the dominant peak for each state in the state space.  The red region shows the region of dominance for $r_b$, the blue region shows the region of dominance for $r_c$, and the green region shows the region of dominance for $r_a$.}
  \label{fig:rewards}
\end{figure}
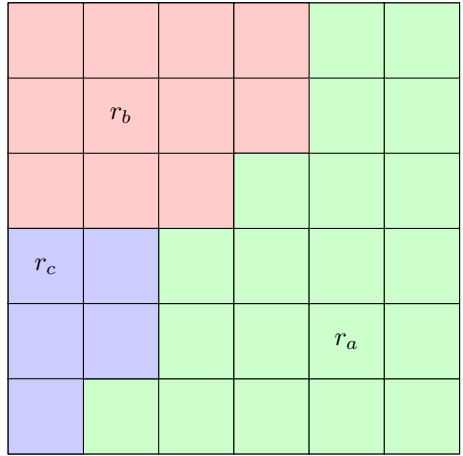

\subsection{Identifying Collected Rewards}

Intuitively, we can see that it is only possible to collect rewards that are in the dominated region that the initial state lies within.  However, we can do better and determine exactly which rewards will and will not be collected from a given initial state.  

\begin{theorem}
Given the dominant peak at a state $s_i$ with a value at that state of $V_{dom}$, any delta peak with a value at $s_i$ of $V_{\Delta} > V_{dom}$ will be collected.  Conversely, any delta peak with a value at $s_i$ of $V_{\Delta} < V_{dom}$ will not be collected.
\end{theorem}

\begin{figure}[H]
\centering
\begin{tikzpicture}[decoration={markings, mark=at position 1.0 with {\arrow{>}};}]
      \tkzDefPoint(2,.05){A}
      \tkzDefPoint(2,.3){B}
      \tkzDrawLines[postaction={decorate}](A,B)
      \tkzLabelPoint[above right](B){$r_d$}
      \tkzLabelPoint[below](A){$s_d$}
      
      \tkzDefPoint(7,.00){A}
      \tkzDefPoint(7,1.3){B}
      \tkzDrawSegments[postaction={decorate}](A,B)
      \tkzLabelPoint[above right](B){$r_i$}
      \tkzLabelPoint[below](A){$s_i$}
      
      \tkzDefPoint(4,.00){A}
      \tkzDefPoint(4,0.2){B}
      \tkzDrawSegments[postaction={decorate}](A,B)
      \tkzLabelPoint[above right](B){$r_k$}
      \tkzLabelPoint[below](A){$s_k$}
      
      \draw[->] (0 ,0) -- (10,0) node[right] {$state$};
      \draw[->] (0 ,0) -- (0,3) node[above] {$value$};
      \draw[thick,domain=0:10,samples=1000,variable=\x,blue] plot ({\x},{2.8*0.9^(abs(\x-7))});
      \draw[thick,domain=0:10,samples=1000,variable=\x,red] plot ({\x},{1.2*0.9^(abs(\x-4))});
      \draw[thick,domain=0:10,samples=1000,variable=\x,green] plot ({\x},{2.1*0.9^(abs(\x-2))});
      \draw[thick,densely dotted,domain=1.65:2.10,smooth,variable=\y,black]  plot (2,{\y});
      \tkzLabelPoint[above right](6.7,3  ){$\mathcal{B}^i$}
      \tkzLabelPoint[above right](3.7,1.25){$\Delta^{k}$}
      \tkzLabelPoint[above right](1.7,2.45){$\Delta^d$}
      
\end{tikzpicture}
\caption{Illustration of baseline peak (blue), a delta peak (red) that will not be collected, and a delta peak (green) that will be collected.}
\end{figure}
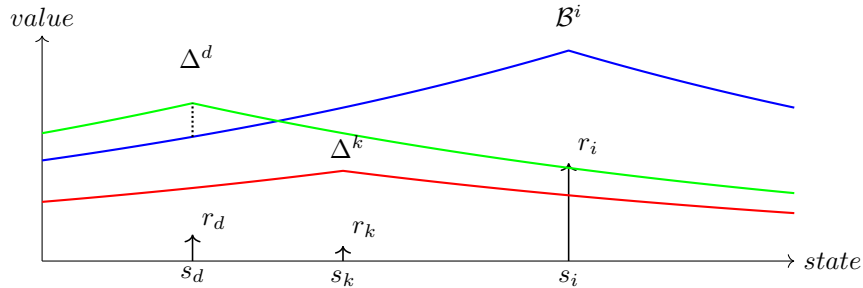

\begin{proof}
In \cite{bertramExact}, it was shown that the optimal value function $V^* \in V^\mathcal{M}$ and that $V^*$ is equal to the element-wise maximum of $V^\mathcal{M}$.  Let us denote as $P^*$ the combination of peaks and the corresponding value functions $M^* \in \mathcal{M}$ that result in the optimal value function $V^*$.  Let us assume that at a state $s_i$ there exists a dominant peak $P_{dom}$ with a value at $s_i$ of $V_{dom}$, and further assume that a delta peak $P_{\Delta}$ with value at $s_i$ of $V_{\Delta}$ such that $V_{\Delta} > V_{dom}$, and finally that if there are more delta peaks, the delta peak $P_{\Delta}$ is the one which has maximum value among them at $s_i$.

Then, it is clear from the definition of $V^*$ that $V_{dom}$ is not the maximum value at $s_i$ and that it is in fact $V_{\Delta}$.  This then implies that the delta peak $P_{\Delta}$ is selected and by definition is collected once along the path to the dominant peak, which may cause a divergence of the optimal path from the path that would result in following the dominant peak directly.  This represents a case where the cost of diverting away from the direct path to the dominant peak is overcome by the benefit of obtaining the reward from the delta peak.

Similarly, if $V_{\Delta} < V_{dom}$, then $V_{dom}$ is the maximum value at $s_i$ and the delta peak will not be collected along the optimal path.  This represents a case where the cost of diverting away from the direct path to the dominant peak is not justified by the collection of the reward.
\end{proof}

With this proof, we now have a way to identify which rewards will be collected.  Given the list of optimal peaks from the Bertram \memless algorithm and an initial state $s_i$, the rewards associated with the peaks listed below are collected as follows.  We denote this set of peaks that are collected as $P^c$ and the corresponding set of rewards $R^c$ as the \emph{collected rewards}.  No other rewards are collected when following the optimal policy from state $s_i$.

\begin{enumerate}
\item the dominant peak (which is either a baseline peak or a combined peak)
\item any delta peaks whose value $V_{\Delta} > V_{dom}$ at state $s_i$.
\end{enumerate}

\subsection{Relative Contribution}

We define the \emph{relative contribution} of a collected peak $p \in P^c$ through the following procedure.

\begin{definition}
Given the set of collected peaks $P^c$ with a length of $k$, we must order them in decreasing order by their value as evaluated at state $s_i$, which we will defined as the list $P^{ord}$ also with length $k$.  The maximum value of this list would then be the first element $P^{ord}(0)$ which is equivalent to $V^*(s_i)$.  We then append to this ordered list a trailing value of 0, denoted as $P^{prepared}$ which then has length $k+1$. The difference in value, $\mathcal{D}$, between the peaks is then defined as:

\begin{equation*}
\mathcal{D}_i = P^{prepared}( i ) - P^{prepared}( i + 1 ), \forall i \in \{ 1, \cdots, k \}
\end{equation*}
\end{definition}

The \emph{relative contribution} of the collected peaks is then the ratio $\frac{\mathcal{D}_i}{V^*(s_i)}$, which could then be expressed as a percentage.  This percentage can be used to determine how strongly a given collected reward is influencing the optimal policy at state $s_i$, which provides a deeper understanding of the action and improves the explainability of the Markov Decision Process optimal policy.  

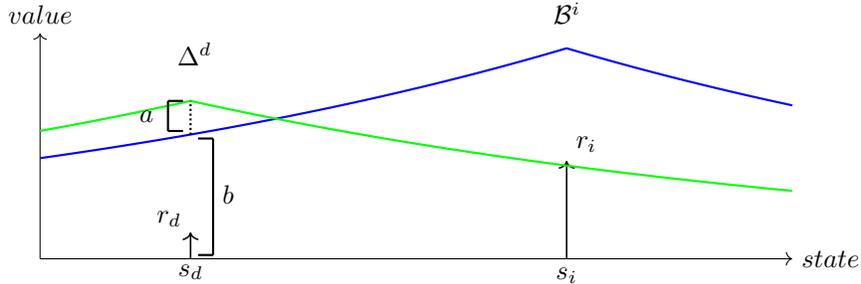
\begin{figure}[H]
\centering
\begin{tikzpicture}[decoration={markings, mark=at position 1.0 with {\arrow{>}};}]
      \tkzDefPoint(2,.05){A}
      \tkzDefPoint(2,.3){B}
      \tkzDrawLines[postaction={decorate}](A,B)
      \tkzLabelPoint[above left](B){$r_d$}
      \tkzLabelPoint[below](A){$s_d$}
      
      \tkzDefPoint(7,.00){A}
      \tkzDefPoint(7,1.3){B}
      \tkzDrawSegments[postaction={decorate}](A,B)
      \tkzLabelPoint[above right](B){$r_i$}
      \tkzLabelPoint[below](A){$s_i$}
      
      \draw[->] (0 ,0) -- (10,0) node[right] {$state$};
      \draw[->] (0 ,0) -- (0,3) node[above] {$value$};
      \draw[thick,domain=0:10,samples=1000,variable=\x,blue] plot ({\x},{2.8*0.9^(abs(\x-7))});
      \draw[thick,domain=0:10,samples=1000,variable=\x,green] plot ({\x},{2.1*0.9^(abs(\x-2))});
      \draw[thick,densely dotted,domain=1.65:2.10,smooth,variable=\y,black]  plot (2,{\y});
      \tkzLabelPoint[above right](6.7,3  ){$\mathcal{B}^i$}
      \tkzLabelPoint[above right](1.7,2.45){$\Delta^d$}
      
      \draw[thick,domain=1.70:2.10,smooth,variable=\y,black]  plot (1.7,{\y});
      \draw[thick,domain=1.7:1.9,smooth,variable=\x,black]  plot ({\x},1.7);
      \draw[thick,domain=1.7:1.9,smooth,variable=\x,black]  plot ({\x},2.1);
      \tkzLabelPoint(1.2,2.1){$a$}
      
      \draw[thick,domain=.05:1.60,smooth,variable=\y,black]  plot (2.3,{\y});
      \draw[thick,domain=2.1:2.3,smooth,variable=\x,black]  plot ({\x},1.6);
      \draw[thick,domain=2.1:2.3,smooth,variable=\x,black]  plot ({\x},.05);
      \tkzLabelPoint(2.3,1.1){$b$}
      
\end{tikzpicture}
\caption{Illustration of baseline peak (blue), and a delta peak (green)  The value at state $s_i$ is $V(s_i) = a + b$, where $b$ is the contribution from the baseline peak and $a$ is the contribution from the delta peak.  The relative contributions are the ratios $\mathcal{D} = \{ \frac{a}{V(s_i)}, \frac{b}{V(s_i)} \}$ from which we can express as a percentage how much each reward source is contributing to the value at the state $s_d$ (or any other state.)}
\end{figure}
\section{Conclusion}

In this paper, we have presented a novel approach to explaining why the optimal policy for a Markov Decision Process selects a specific action, relating the action to the degree in which they are driven by various reward sources.  This reduces the opaqueness of Markov Decision Processes and can be used to analyze the state space to determine which regions of the state space will be attracted to given local maximums of the value function.

This algorithm is based on the research and methods proposed in \cite{bertramExact,bertramMemoryless} and is therefore subject to the same restricted class of Markov Decision Processes.  If the methods can be expanded to work on a more general class of MDPs, then the method described in this paper should also be applicable to this more general class of MDPs.

\bibliographystyle{unsrt}
\bibliography{refs}

\end{document}